\documentclass[letterpaper,10pt,conference]{ieeeconf}

\IEEEoverridecommandlockouts
\newif\iffacetfull
\ifdefined\FACETFULL
  \facetfulltrue
\else
  \facetfullfalse
\fi

\usepackage{amsmath,amssymb}
\usepackage{booktabs}
\usepackage{graphicx}
\usepackage{flafter}
\usepackage{placeins}
\usepackage{needspace}
\usepackage{xcolor}
\usepackage{algorithm}
\usepackage{algpseudocode}
\usepackage[hidelinks]{hyperref}
\hypersetup{pdftitle={SCET-CBF: Speed-Calibrated Soft Guidance for One-Step Dynamic Navigation},pdfauthor={Xiaoyu Han and Dongkun Han}}
\makeatletter
\def\abstract{\normalfont%
  \if@twocolumn%
    \@IEEEabskeysecsize\bfseries\textit{Abstract}:\,%
  \else%
    \begin{center}\vspace{-1.78ex}\@IEEEabskeysecsize\textbf{Abstract}\end{center}%
    \quotation\@IEEEabskeysecsize%
  \fi\@IEEEgobbleleadPARNLSP}
\makeatother
\iffacetfull
  \usepackage{capt-of}
\fi

\newtheorem{theorem}{Theorem}
\newtheorem{proposition}{Proposition}

\newtheorem{assumption}{Assumption}
\newcommand{\CapRestricted}{49.5}

\newcommand{\Uh}{\mathcal{U}_{\mathrm{hard}}}
\newcommand{\Bs}{\mathcal{B}_{\mathrm{safe}}}

\newcommand{\clip}{\operatorname{clip}}

\newcommand{\betamax}{\beta_{\max}}
\newcommand{\dt}{\Delta t}

\newcommand{\fulldoi}[1]{\unskip\iffacetfull, doi: #1\fi}

\title{\LARGE \bf SCET-CBF: Speed-Calibrated Soft Guidance\\for One-Step Dynamic Navigation}
\author{Xiaoyu Han$^{1}$ and Dongkun Han$^{2}$%
\thanks{$^{1}$Xiaoyu Han is with the Department of Physics and Astronomy,
University College London, London WC1E 6BT, United Kingdom.
\texttt{xiaoyu.han.24@ucl.ac.uk}}%
\thanks{$^{2}$Dongkun Han is with the Department of Mechanical and Automation
Engineering, The Chinese University of Hong Kong, Shatin, Hong Kong SAR, China.
\texttt{dkhan@mae.cuhk.edu.hk}}}
\begin{document}
\maketitle\thispagestyle{empty}\pagestyle{empty}

\begin{abstract}
Control barrier functions (CBFs) support collision avoidance, yet dense
dynamic encounters can leave a vehicle with no feasible command.
Geometric refinements and soft constraints help preserve motion choices.
Using these choices also requires accounting for the vehicle's changing
turning response. For a steering-limited bicycle, braking can lower
closing speed while slowing the lateral motion needed before an obstacle
arrives. We propose SCET-CBF, a speed-calibrated soft guidance law for
one-step dynamic navigation. A short lateral escape maneuver gives an
allowable approach speed. Its speed derivative allows a quadratic program
to weigh both effects of braking when choosing acceleration and steering.
Soft penalties accommodate competing obstacle preferences within the
input set permitted by segment and braking constraints. We derive the
guidance and establish protection over each accepted motion segment,
together with conditional braking invariance for a single obstacle.
On 2500 paired tasks with a 60 s deadline, SCET achieves 72.8\% completion,
compared with 49.8\% for CBF/velocity-obstacle control and 69.8\% for an
enhanced 1 s CBF-MPC. At 100 obstacles, SCET and the enhanced MPC complete
52.4\% and 33.6\% of tasks. Paired studies identify the contributions of
speed feedback and early activation. Further evaluations examine the
operating range under lateral acceleration limits and bounded motion
and execution errors.
\end{abstract}
\noindent{\footnotesize Code and reproduction materials:\par
\url{https://anonymous.4open.science/r/SCET-CBF-0C0F/}}\par

\begin{figure}[t]\centering
\includegraphics[width=\columnwidth]{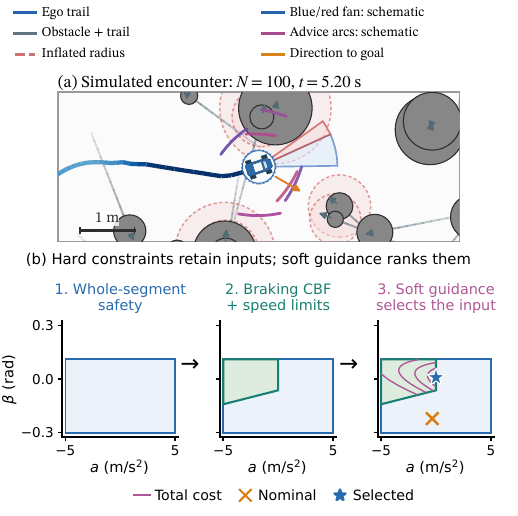}
\caption{From available commands to continued passage.
(a) The fan and arcs illustrate the car's steering choices and
turning-based preference. (b) Segment geometry gives the blue input domain;
speed and braking retain the green set. Soft guidance selects the star
within it. The orange cross is the goal-seeking input.}
\label{fig:mechanism}
\end{figure}

\section{Introduction}
Dense moving obstacles turn navigation into coupled decisions. Clearing
one encounter can leave little room for the next. Velocity obstacles
(VOs) identify dangerous relative
motion, while control barrier functions (CBFs) express safety conditions
as input constraints~\cite{velocityobstacle,mavo}. With bounded controls,
competing constraints can exhaust the available commands before goal entry.

Recent advances preserve more choices. Multi-agent CBF/VO control (MAVO)
allows VO guidance to yield while retaining hard braking
constraints~\cite{mavo}. Dynamic parabolic CBF (DPCBF) adjusts its boundary
with distance and relative speed to reduce collision-cone
restriction~\cite{dpcbf}. The question is now how to use the remaining
choices to sustain passage through successive encounters.

A steering-limited bicycle makes this question pressing. Lateral motion
requires forward speed, and moving obstacles impose a time window for
evasion. Braking can lower closing speed while reducing yaw rate and
lateral response, which scale with $v$ and $v^2$. Rebuilding speed takes
time under bounded acceleration. Slowing for one obstacle can therefore
leave too little response for the next. Figure~\ref{fig:allowance}a
illustrates this time-dependent response.

We develop SCET-CBF (Speed-Calibrated guidance within Exact Turning sets
with a control barrier function) for this task. E denotes \emph{Exact}:
the complete segment-safe steering set under the sampled model.
Figure~\ref{fig:mechanism} shows segment checking, speed/braking constraints,
and soft selection through a quadratic program (QP).

Our main contribution is an \emph{approach allowance} derived from a short
bicycle escape model. It quantifies the accommodated closing speed.
The QP retains its speed derivative to weigh both effects of braking.
Building on MAVO's soft/hard separation~\cite{mavo},
we use this allowance as relaxable obstacle guidance within the segment
and braking constraints. Competing preferences can yield while preserving
all hard-feasible commands. Early activation gives turning time to develop.

Segment protection and a bicycle braking bound support the design.
Paired studies show passage gains and millisecond computation, and
examine guidance mechanisms and robustness.
Outcomes are goal entry (\emph{success}), disc overlap (\emph{contact}),
no executable command (\emph{exhaustion}), or deadline expiry (\emph{timeout}).
Sections II to V present the model, method, analysis and evaluation.
\section{Model and Design Context}
\subsection{Vehicle and execution model}
For position $p=(x,y)$, heading $\theta$, speed $v$, and input
$u=(a,\beta)$, the relaxed bicycle obeys
\begin{align}
\dot x&=v\cos\theta-v\sin\theta\,\beta,&
\dot y&=v\sin\theta+v\cos\theta\,\beta,\nonumber\\
\dot\theta&=v\beta/\ell_r,&\dot v&=a.\label{eq:plant}
\end{align}
We use $\dt=0.05$ s Euler updates and linear motion between successive
positions. Steering $\beta$ determines the current displacement;
acceleration $a$ changes the next speed. Bounds are $|a|\le5$ m/s$^2$, $|\beta|\le0.302825$ rad,
$v\in[0.2,3.5]$ m/s, and $\ell_r=\ell_f=0.2$ m.
The DPCBF benchmark's positive speed floor excludes parking. The input $\beta$ is a directly commanded sideslip surrogate;
Fig.~\ref{fig:bicycle} relates it to the physical wheel angle.
Acceleration and steering act simultaneously within independent
bounds~\cite{polack2017bicycle}. In this steering-limited model, speed
determines how much lateral motion develops during an encounter.
The spatial curvature at fixed steering is independent of speed.
Section~V examines how steady-turn limits, obstacle prediction errors
and actuator response affect navigation.

\begin{figure}[t]\centering
\includegraphics[width=.70\columnwidth]{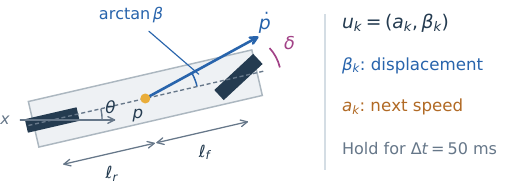}
\caption{The vehicle behind the guidance model. Wheel angle $\delta$
maps to surrogate $\beta$; velocity points along
$\theta+\arctan\beta$ in~\eqref{eq:plant}.}
\label{fig:bicycle}
\end{figure}

Obstacle $j$ has position $p_j$, velocity $w_j$, and radius $r_j$; the
vehicle radius is 0.3 m. SCET, MAVO, rFF-CBF, and DPCBF use the
same bounded proportional goal controller
for their nominal input $u_{\rm nom}$.
With goal distance $d_g$ and wrapped bearing error $e_g$, set
$\bar d_g=\max(d_g-0.05,0.05)$ m and
$v_d=\clip(0.5\bar d_g\max(\cos e_g,0),v_{\min},v_{\max})$.
The nominal acceleration is $\operatorname{sat}_{5}(1.5(v_d-v))$;
nominal steering is
$\arctan[\tfrac12\tan(\operatorname{sat}_{32^\circ}(0.5e_g))]$,
where $\operatorname{sat}_b$ clips to $[-b,b]$ and gains use SI units.

\subsection{Design origins and closest methods}
VOs describe collision geometry in relative velocity
space~\cite{velocityobstacle}. MAVO uses this geometry as a soft objective
and a hard braking CBF for protection~\cite{mavo}. Its braking argument
uses double-integrator dynamics; it also evaluates a car, with MATLAB
implementations for both. DPCBF adapts a hard parabolic boundary with
distance and relative speed~\cite{dpcbf}. SCET builds on these designs
with a soft approach allowance derived from a bicycle maneuver.

Recent SC3BF~\cite{sc3bf} shifts a hard collision-cone CBF with a
distance- and speed-dependent allowance and provides bounded-input
feasibility and invariance conditions for a continuous-time single-obstacle
bicycle model. SCET derives its allowance and activation time from a short
bicycle escape maneuver, retains the allowance's speed derivative, and uses
the result as soft guidance among multiple obstacles. Exact sampled-segment
checking and braking constraints supply its protective layer; conflicting
guidance preferences can relax.

Turning-circle CBFs bring maneuverability into MPC~\cite{turningcircle}.
Their QP variant retains the speed-dependent radius derivative for
independent yaw control~\cite{turningcircleqp}. Future-focused CBFs predict
closest approach~\cite{ffcbf}, and clearance curves link speed to evasive
motion~\cite{clearance1998}. SCET expresses maneuver information as a
scalar allowance and tests its speed sensitivity in a soft QP.
Positive-speed VO/hybrid CBFs~\cite{positivevo,positivehybrid},
sampled/segment-safe CBFs~\cite{sampledcbf,segmentsafecbf}, and predictive
CBFs~\cite{dcbfmpc} complement this design through motion constraints,
execution protection, and future planning.

\section{The SCET-CBF Update}
Segment geometry and braking first define the available input set.
Soft guidance then ranks its choices using the vehicle's turning response
(Fig.~\ref{fig:mechanism}b and Algorithm~\ref{alg:scet}).

\subsection{Step 1: steering commands that clear the next segment}
The car and each obstacle move between updates. \emph{Segment safety}
checks their separation throughout that motion. For $e_\theta=(\cos\theta,\sin\theta)$ and
$n_\theta=(-\sin\theta,\cos\theta)$, define relative endpoints and clearance
\begin{align}
r_j^0&=p_j-p,\nonumber\\
r_j^1(\beta)&=r_j^0+\dt(w_j-ve_\theta-vn_\theta\beta),\nonumber\\
g_j(\beta)&=\min_{0\le\lambda\le1}\|r_j^0+\lambda(r_j^1-r_j^0)\|^2-(R_j^c)^2.
\label{eq:sampled-clearance}
\end{align}
The common time $\lambda$ synchronizes the motions. The command set is
$\Bs^0=\{\beta\in[-\betamax,\betamax]:g_j(\beta)\ge0,\ \forall j\}$.

To compute the set, write $d(\beta)=r_j^1-r_j^0$,
$A=d^\top d$, $B=(r_j^0)^\top d$, and
$C=\|r_j^0\|^2-(R_j^c)^2$. For $A>0$, the closest point occurs at
$\lambda^*=\operatorname{clip}(-B/A,0,1)$; for $A=0$, set $\lambda^*=0$.
Then $g_j=C+2\lambda^*B+(\lambda^*)^2A$.
Branch or safety changes occur only at roots of
$A,B,A+B,A+2B+C,AC-B^2$, all at most quadratic in $\beta$.
Sorting these roots and testing each intervening interval and boundary
therefore recovers the entire continuous steering domain for optimization.

For a bounded constant steering error $e_\beta$, erode every component
$[\underline\beta_m,\overline\beta_m]$ to
$[\underline\beta_m+e_\beta,\overline\beta_m-e_\beta]$, discarding empty
intervals. Inflate obstacle radii by a bound $\rho_j$ on relative path
error over the whole sample. Initially $R_j^c=1.05(r_{\rm ego}+r_j)+\rho_j$.
If the set is empty, remove the optional 5\% buffer and retry with
$R_j^c=r_{\rm ego}+r_j+\rho_j+10^{-4}$ m. The same physical-radius cushion
is checked on every return. Braking and guidance retain the 5\% radius.

\begin{assumption}[Current coverage and execution]
Every obstacle within 15 m is observed with the stated error bounds, and
obstacle speed is bounded by $\bar v_o$ (1.2 m/s in this benchmark). Every obstacle within
$R_j^c+\dt(v\sqrt{1+\betamax^2}+\bar v_o)$ is included independently of
rank. This reach bound is below the sensing range. The actual ego and
obstacle paths obey the error bounds used to construct $\rho_j$.\hfill$\square$
\end{assumption}

\subsection{Step 2: speed limits and braking margin}
Braking reduces radial approach before separation is spent. Its margin
jointly constrains acceleration and steering within the steering set.
Let $d_j=\|p_j-p\|$, $e_j=(p_j-p)/d_j$, and
$n_j=(-e_{j,y},e_{j,x})$. Define radial/tangential relative speeds
$x_j=e_j^\top(w_j-ve_\theta)$, $y_j=n_j^\top(w_j-ve_\theta)$, and
$c_j=[-x_j]_+$. With $r_j^c=1.05(r_{\rm ego}+r_j)$, use MAVO's braking form
\begin{equation}
h_j^B=d_j-r_j^c-c_j^2/(2a_B),\qquad a_B=0.055\ \mathrm{m/s^2}.
\label{eq:bicycle-cbf}
\end{equation}
The margin compares separation with the distance needed to reduce radial
approach. Theorem~\ref{thm:pairwise} derives $a_B$ as a conservative lower
bound on combined radial control capability. The longitudinal braking
limit remains 5 m/s$^2$. Using relative velocity covers a constant-velocity
moving obstacle, including a static obstacle as a special case.

For $x_j<0$, let $\cos\vartheta_j=e_j^\top e_\theta$ and
$\sin\vartheta_j=n_j^\top e_\theta$. Differentiating the full bicycle gives
\begin{align}
L_fh_j^B&=-c_j+(c_j/a_B)y_j^2/d_j,\nonumber\\
C_{j,a}&=-(c_j/a_B)\cos\vartheta_j,\nonumber\\
C_{j,\beta}&=v\sin\vartheta_j+\frac{c_jv}{a_B}
\left[\frac{v}{\ell_r}\sin\vartheta_j-\frac{y_j}{d_j}\cos\vartheta_j\right].
\label{eq:brake-row}
\end{align}
For $x_j\ge0$, use $L_fh_j^B=x_j$, $C_{j,a}=0$,
$C_{j,\beta}=v\sin\vartheta_j$. The CBF condition~\cite{ames2017cbfqp} gives
\begin{equation}
C_{j,a}a+C_{j,\beta}\beta+L_fh_j^B+\alpha_Bh_j^B\ge0,
\quad\alpha_B=0.005\ \mathrm{s^{-1}}.\label{eq:hard}
\end{equation}
Among the processed obstacles, enforce all rows satisfying
$h_j^B\le H_B+\eta$, where $H_B=\min_jh_j^B$ and $\eta=0.02$ m.
Thus nearly tied worst margins are protected together. Processing includes
the nearest ten obstacles within 15 m and those with $d_j<r_j^*$, where
$r_j^*=r_j^c+\Delta v^2/(2a_{\max})+\bar v_o\Delta v/a_{\max}
+\dt(v+\bar v_o)$ and $\Delta v=[v-v_{\min}]_+$.
This deceleration-based selection rule supplements Assumption~1.

For each steering interval, intersect~\eqref{eq:hard} with
\begin{equation}
\mathcal A(v)=[-a_{\max},a_{\max}]\cap
[(v_{\min}-v)/\dt,(v_{\max}-v)/\dt].\label{eq:speed-input}
\end{equation}
The resulting intersections form the hard input set $\Uh$.
When a braking margin is negative, its row requests recovery.
Section~IV establishes the pairwise invariance conditions.

\subsection{Step 3: selecting for continued maneuverability}
Guidance asks how much approach the car can accommodate while turning.
Figure~\ref{fig:allowance} links the maneuver allowance, its speed
sensitivity, and the time needed for the turn to develop.

\paragraph{From turning capacity to approach allowance}
Following the clearance-curve link between speed and
avoidance~\cite{clearance1998}, maximum steering gives yaw rate
$\omega=v\betamax/\ell_r$.
A quarter turn takes $t_q=\pi/(2\omega)$. Integrating~\eqref{eq:plant}
gives lateral displacement $\ell_r/\betamax+\ell_r$: the first term
comes from heading rotation, and the second from direct sideslip.
We use the heading-rotation contribution to set the proxy scale.
Matching this displacement to $a_{\rm eff}t_q^2/2$ gives
\begin{equation}
a_{\rm eff}(v)=\frac{8}{\pi^2}\frac{v^2\betamax}{\ell_r}.
\label{eq:aeff}
\end{equation}
The resulting frontal escape is $x(t)=\rho-st$, $y(t)=a_\perp t^2/2$, where $\rho$
is initial center separation, $s$ is approach speed, and $a_\perp$ is
constant lateral acceleration. Tangency to avoidance radius $r$ gives
\begin{align}
x_c&=(\sqrt{\rho^2+8r^2}-\rho)/2,\nonumber\\
s_0(\rho,r,a_\perp)&=\frac{\rho-x_c}{\sqrt{2\sqrt{r^2-x_c^2}/a_\perp}},\qquad \rho>r>0.
\label{eq:vertex}
\end{align}
Indeed, the contact and stationarity equations imply
$x_c^2+\rho x_c-2r^2=0$. The squared distance has
$F''(t)=2s^2+3a_\perp^2t^2>0$, so this tangency is the unique minimum. 

\begin{figure}[t]\centering
\includegraphics[width=\columnwidth]{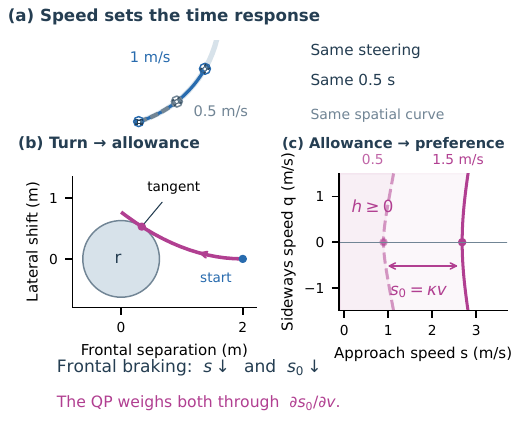}
\caption{How turning response becomes a control preference.
(a) Equal steering follows the same continuous-model curve, with different
speeds reaching different positions after 0.5 s. (b) A tangential escape
sets the proxy allowance $s_0$. (c) Speed moves the soft $h=0$ boundary in
relative-velocity coordinates ($d=2$ m, $r^c=0.63$ m). Braking changes both
approach and allowance; the QP accounts for both changes.}
\label{fig:allowance}
\end{figure}

To include nonfrontal encounters, let $s_j=-x_j$ be approach speed,
$q_j=y_j$ sideways relative speed, and $V_j=\|w_j-ve_\theta\|$. For
$d_j>1.02r_j^c$, use
\begin{equation}
h_j=s_0(d_j,r_j^c,a_{\rm eff})+
\kappa_q\sqrt{d_j^2-(r_j^c)^2}\frac{q_j^2}{\max(V_j,10^{-6})}-s_j.
\label{eq:guidance}
\end{equation}
The second term credits sideways relative motion already helping the
encounter. Positive $h_j$ means the approach lies within the combined
allowance.
\paragraph{How speed changes the preference}
Braking changes both approach speed and the turning allowance.
The QP needs both changes to judge how slowing affects the preference
(Fig.~\ref{fig:allowance}c). At fixed distance and radius,
$s_0=\kappa(d_j,r_j^c)v$ for $v>v_{\min}$, so
$\partial s_0/\partial v=s_0/v$.

For a stationary frontal obstacle, $q=0$ and $h=(\kappa-1)v$;
its acceleration coefficient is $\kappa-1$.
At $d=2$ m and $r^c=0.63$ m, $\kappa=1.786$ gives $+0.786$:
the allowance grows faster than the approach speed during acceleration.
At $d=0.8$ m, $\kappa\approx0.347$ reverses the preference.
Thus guidance can favor acceleration during approach when the hard set
permits it, and favor slowing as the encounter tightens.

\paragraph{When guidance acts}
Turning takes time. The soft row starts acting before the current
approach preference is lost.
Let $A_j=L_gh_j$ be the margin's input sensitivity
and $L_fh_j$ its drift. With $\alpha=10$ s$^{-1}$ and permitted violation $\xi_j\ge0$,
the soft row is
\begin{align}
&A_ju+L_fh_j+\alpha(h_j-\varepsilon_j)+\xi_j\ge0,\nonumber\\
&\varepsilon_j=\dt\bigl(|L_fh_j|+|A_{j,a}|a_{\max}
+|A_{j,\beta}|\betamax\bigr)\nonumber\\
&\quad+[-L_fh_j]_+T_j^{\rm adv}.\label{eq:advice-row}
\end{align}
The first term of $\varepsilon_j$ estimates one update of margin change
from local derivatives. The second brings guidance forward when the
margin is decreasing. We use the physical maneuver time
$T_j^{\rm adv}=t_j^{\rm turn}=\sqrt{2r_j^c/a_{\rm eff}}$.
Slower turning gives an earlier warning. Geometry and speed determine
this advance; the task deadline specifies the evaluation duration.

Approaching obstacles with $c_j>10^{-3}$ m/s receive time-to-collision
weight $\omega_j=2c_j/d_j$. 

The maneuver model sets the allowance and turning time. The fixed choices
$\kappa_q=0.1$ m$^{-1}$, $\alpha=10$ s$^{-1}$, and the multiplier 2 in
$\omega_j$ set lateral credit, activation rate, and guidance strength.
The 60 s study uses these same parameters. Soft weights set the relative
importance of preferences over the hard input set.

\paragraph{Numerical evaluation}
At
$d_j\le1.02r_j^c$ the implementation uses $h_j=-1$ m/s to avoid invalid
tangency queries. A shared central-difference routine evaluates all guidance variants and
branches with step $2\times10^{-5}$ in SI units. The speed passed to
$a_{\rm eff}$ is $\max(v,v_{\min})$; the switch at $1.02r_j^c$ is unsmoothed.
Section~V-F discusses these definitions at the speed floor and near obstacles.

\begin{algorithm}[t]
\caption{Default SCET-CBF update}\label{alg:scet}
\small\begin{algorithmic}[1]
\Require Ego state, covered obstacle observations, goal, error bounds
\State Form $u_{\rm nom}$ and $\mathcal A(v)$.
\State Compute segment-safe steering intervals; inflate for path errors and erode for steering error.
\State If empty, retry without the optional 5\% buffer, retaining errors and the physical cushion.
\If{no interval remains}\State \Return sampled-set exhaustion\EndIf
\State Form near-minimum braking rows and independent soft rows.
\For{each steering interval}
\State Intersect with speed and braking constraints to obtain $P_m$.
\State If nonempty, solve~\eqref{eq:qp}; on solve failure project the nominal input and validate it.
\EndFor
\State Return the lowest-cost validated candidate after bounds, hard-row and segment checks; otherwise report exhaustion or numerical failure.
\end{algorithmic}
\end{algorithm}

\subsection{Soft guidance within the hard input set}
Different encounters can call for conflicting responses. Independent
slacks let the QP balance these preferences with progress toward the goal.
On each
nonempty component $P_m$ of $\Uh$, solve
\begin{equation}
\min_{u\in P_m,\,\xi\ge0}
\|D_u(u-u_{\rm nom})\|^2+\sum_j\omega_j\xi_j^2,
\quad A_ju+b_j+\xi_j\ge0,\label{eq:qp}
\end{equation}
where $D_u=\operatorname{diag}(a_{\max}^{-1},\betamax^{-1})$ and
$b_j=L_fh_j+\alpha(h_j-\varepsilon_j)$. 

For any $u\in\Uh$, choosing $\xi_j=[-A_ju-b_j]_+$ satisfies each soft row.
Every hard input therefore remains available. Its slack cost measures
how strongly the corresponding maneuver preference opposes it.

Eliminating the slacks makes this mechanism explicit. At the current
state, the equivalent objective and its gradient are
\begin{align}
J(u)&=\|D_u(u-u_{\rm nom})\|^2+\sum_j\omega_j[-A_ju-b_j]_+^2,\nonumber\\
\nabla J(u)&=2D_u^\top D_u(u-u_{\rm nom})
-2\sum_j\omega_j\xi_j A_j^\top.\label{eq:soft-gradient}
\end{align}
A violated preference pulls the input along $A_j$; its violation and
urgency set the strength. A satisfied preference adds no correction.
These pulls form the magenta contours in Fig.~\ref{fig:mechanism}b.
Speed feedback changes their acceleration component even at the same
margin value. Figure~\ref{fig:guidance} follows this change into the car's motion.

Each $P_m$ is a convex polygon. Comparing its optimized cost with those
of all other nonempty components selects the best input over the union.

\section{Analysis of the Hard Layer}
\subsection{Protection of the current segment}
\begin{theorem}[Accepted segment]\label{thm:sample}
Under Assumption~1, a returned command remains outside the physical
obstacle discs throughout the current sample, provided its final check
includes the stated error bounds and positive physical cushion.
\end{theorem}
\noindent\emph{Proof.} Equation~\eqref{eq:sampled-clearance} checks the entire
relative segment. Erosion covers steering error, and the path-error
inflation ensures physical clearance. Assumption~1's reach bound excludes
contact with omitted obstacles. Concatenating accepted segments preserves
clearance. $\square$

\needspace{9\baselineskip}
\begin{proposition}[Complete current command set]\label{prop:complete}
With exact arithmetic, zero path and steering errors, and
$\mathcal A(v)\ne\varnothing$, the returned interval union is empty exactly
when no held input clears the declared radius throughout the sample.
\end{proposition}
\noindent\emph{Proof.} Root partitioning covers every sign and branch change,
including the boundary points; current position is independent of $a$.
Thus no admissible steering is omitted. $\square$

The double-precision implementation rechecks returned commands.

\subsection{Conditional braking margin invariance}
\begin{theorem}[Pairwise bicycle CBF]\label{thm:pairwise}
For one constant-velocity obstacle with $r_j^c\ge0.4$ m, the continuous
plant~\eqref{eq:plant}, $v\ge0.2$ m/s, and the symmetric input box,
Eq.~\eqref{eq:hard} is feasible on $\{h_j^B\ge0\}$ for
$0<a_B<0.0600$ m/s$^2$. A locally Lipschitz feedback satisfying this row
preserves the set while these domain and input conditions hold.
\end{theorem}
\noindent\emph{Proof.} For $c_j>0$, maximize $\dot h_j^B/c_j$ over the
input box. Split orientations at $|\cos\vartheta_j|=\chi$: acceleration
acts where the heading has sufficient radial projection, and steering
acts near transverse orientations. In the first sector, acceleration
supplies $a_{\max}\chi$. In the other, steering and completion of the square in
$|y_j|$ supply
$v_{\min}^2[\betamax\sqrt{1-\chi^2}/\ell_r-
\betamax^2\chi^2/(4R_{\min})]$.
With $R_{\min}=0.4$ m, taking $\chi=0.012$ gives lower bounds $0.0600$ and $0.06056$ m/s$^2$.
Both exceed $a_B=0.055$, so the maximal margin derivative is positive.
The proof parameter
$\chi$ balances these two lower bounds; it is absent from the controller
and is determined here from the vehicle limits.
When $x_j\ge0$, zero steering gives $\dot h_j^B=x_j\ge0$.
The CBF comparison argument establishes invariance. $\square$

The bound ties the braking parameter to the bicycle's input limits,
including its least favorable low-speed orientations. The controller
also imposes next speed limits and multiple obstacle rows; Section~V
evaluates their joint feasibility.

\subsection{Optional continuation after task termination}
Committed continuation follows gatekeeper and backup CBF
principles~\cite{gatekeeper,backupcbf}. State $z$ includes time,
vehicle/actuator state, queued commands and hold phase. A valid plan $W$
satisfies input and collision constraints through its maneuver and an
invariant terminal mode.

\begin{proposition}[Committed safety, optional mode]\label{prop:optional}
From a valid $(z,W)$, accept an input only after committing its verified
execution and a valid successor plan. On rejection or task termination,
execute the saved plan. Validity is preserved under the verifier's bounds
if unexecuted suffixes remain available and replacements match their time
and execution-state anchors.
\end{proposition}
\noindent\emph{Proof.} A valid replacement or the saved suffix covers
execution, including computation; its invariant terminal mode continues
safety. $\square$

\section{Evaluation}
We test sustained passage, guidance mechanisms and operating range.
\begin{figure*}[t]\centering
\includegraphics[width=.98\textwidth]{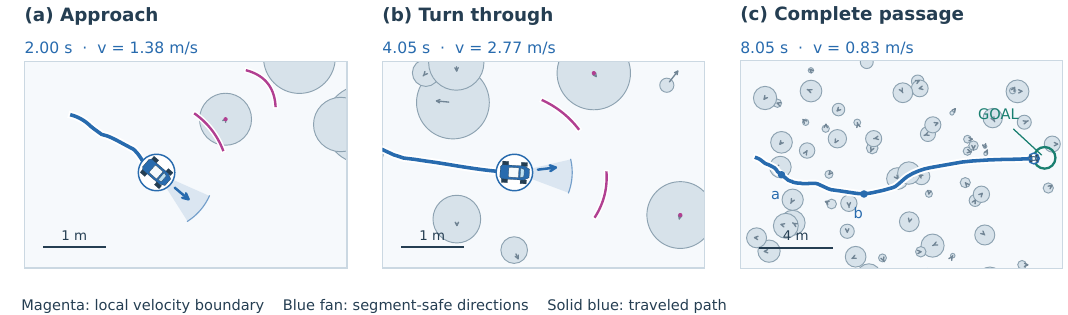}
\caption{Local decisions build a passage through 100 moving obstacles.
(a,b) Enlarged views show the car, motion and guidance. (c) The path reaches
the goal at 8.05 s. Blue gives segment-safe directions and the traveled
path. Magenta shows up to three nearest eligible soft rows, evaluated at
the displayed state (0.22 m per m/s).
Figure~\ref{fig:primary} tests full-cohort passage.}
\label{fig:passage}\end{figure*}

\subsection{Protocol and direct comparisons}
Each controller faces 500 shared seeds at each
$N\in\{10,40,60,80,100\}$, totaling 2500 scenarios. The ego starts at
$(0,0,0,0.5)$ and has 60 s to enter a 0.6 m goal ball around $(16,0)$.
Obstacle centers, radii, headings and speeds are uniform on
$[2,14]\times[-7,7]$ m, $[0.1,0.7]$ m, $[0,2\pi)$ and $[0,1.2]$ m/s.
We independently generate denser DPCBF-inspired scenes in 168 m$^2$.
At fixed obstacle count, initial density is 10.7 times that of our
1800 m$^2$ DPCBF-style reconstruction. Obstacles subsequently disperse
on an unbounded plane.

We first compare one-step CBF-QP designs (MAVO, DPCBF and rFF-CBF),
then predictive controllers.
All adaptations share the plant, goal, input bounds, observations and
whole segment contact evaluator. Our Python MAVO port follows the
published braking/soft-VO formulation and disc time-to-collision weighting.
Weight 100 is selected on separate development scenarios. The public MATLAB
car example differs in steering, goal policy and parameters. DPCBF calls
the authors' \texttt{safe\_control} analytic barrier through the shared
interface. The rFF-CBF port follows the relaxed future-focused
formula~\cite{ffcbf}, with a 5 s closest-approach window and CBF gain 10.
Sources and configurations accompany the code; SCET retains its weights.
Failed MAVO/SCET solves invoke checked hard-domain
projection. DPCBF and rFF-CBF solve polygon QPs directly; MPC returns no command.

The original DCBF-MPC uses ten steps (0.5 s), decay 0.8 and 250 IPOPT iterations,
without a binding wall-time limit. Its stage weights are
$Q=\operatorname{diag}(50,50,1,2)$, $R=\operatorname{diag}(0.5,2)$;
terminal weight is $2Q$, with no terminal constraint. The reference is
$(16,0,0,1)$; shifted previous solutions replace the initial zero-input
warm start. SCET uses OSQP tolerances $10^{-7}$.
A 250-world comparison with matched IPOPT tolerances and iteration budgets
gives 190/105 SCET/MPC completions; SCET also completes 190 with OSQP.
Comparators process the nearest ten obstacles within 15 m; SCET adds
urgent obstacles for coverage. Restricting all SCET channels to the same ten
obstacles gives 1829/2500 completions.

A separate 200-world study tests 10/20/40-step horizons and increased
solve effort. It selects a 20-step (1 s) MPC with 1000 iterations and two
deterministic turn/brake initializations after failure, denoted $H20+R$.
The configuration is fixed before the 2500-world comparison, retaining
the original cost, dynamics, observations and decay.
TC-MPC transfers the authors' two-circle geometry and relaxable prediction
rows to the common bicycle and $H20+R$ solver. On 20 separate development
worlds, six hard/soft and decay settings select soft rows with decay
$1-\sqrt{0.95}$; linear and quadratic slack weights are $10^6$.
Native author-source unicycle runs independently recover passage in
static, head-on and overtaking cases.

Serial timing on ten shared worlds per method (Apple M4, one BLAS
thread) gives full-call P50/P95 of 1.77/2.11 ms for SCET, 13.46/25.59 ms
for $H10$ MPC, 57.49/125.57 ms for $H20+R$, and 51.39/124.79 ms for
TC-MPC, excluding initialization.
Differences use percentage points (pp). Brackets give pointwise 95\%
intervals from 20000 whole-seed bootstrap draws, keeping five densities
together. Exact seed-label tests use Holm correction across ten primary
controller/deadline contrasts.

\subsection{Sustaining passage as obstacle density increases}
\paragraph{One-step CBF-QP navigation}
SCET completes 72.8\%, followed by MAVO at 49.8\%, rFF-CBF at 30.6\%
and DPCBF at 23.6\% (Table~\ref{tab:outcomes}). The SCET-minus-MAVO difference
is +23.0 pp [20.7,25.2]. At $N=100$, the four methods complete
262/73/6/8 of 500 tasks (Fig.~\ref{fig:primary}a).
The separation grows as encounters become denser.
Of SCET's successes, 89.4\% arrive by 30 s; median arrival is
14.8 s (interquartile range 10.45 to 22.0 s).

\begin{table}[t]\centering\footnotesize
\caption{One-step CBF-QP methods: 2500 paired scenarios.
S: goal; C: contact before termination; E: no executable command, including
failed solves; T: task timeout. Both deadlines evaluate the same trajectories.}
\label{tab:outcomes}\setlength{\tabcolsep}{4.4pt}
\begin{tabular}{lrrrrr}\toprule
&30 s&\multicolumn{4}{c}{60 s}\\
Controller&S&S&C&E&T\\\midrule
SCET-CBF &1627&1820&0&670&10\\
MAVO &1134&1246&0&1182&72\\
rFF-CBF adapted &762&764&8&1728&0\\
DPCBF analytic &458&589&15&1893&3\\
\bottomrule\end{tabular}\end{table}

\paragraph{Comparison with predictive control}
Enhanced MPC raises original MPC completion from 45.0\% to 69.8\%
(Table~\ref{tab:predictive}). SCET gains +3.0 pp [0.8,5.2] overall and
+18.8 pp [13.0,24.6] at $N=100$ (262 versus 168 completions).
At 30 s, SCET/enhanced MPC complete 1627/1745.
TC-MPC reaches 73.5\% with 646 contact terminations; its completion
difference from SCET is 0.7 pp [$-1.5$,2.8]. SCET protects every returned segment.

\begin{table}[t]\centering\footnotesize
\caption{SCET and predictive control: the same 2500 scenarios,
outcomes and deadlines as Table~\ref{tab:outcomes}.}
\label{tab:predictive}\setlength{\tabcolsep}{4.4pt}
\begin{tabular}{lrrrrr}\toprule
&30 s&\multicolumn{4}{c}{60 s}\\
Controller&S&S&C&E&T\\\midrule
SCET-CBF &1627&1820&0&670&10\\
MPC, $H10$ &1126&1126&0&1374&0\\
MPC, $H20+R$ &1745&1745&0&754&1\\
TC-MPC adapted &1556&1837&646&0&17\\
\bottomrule\end{tabular}\end{table}

\begin{figure}[t]\centering
\includegraphics[width=\columnwidth]{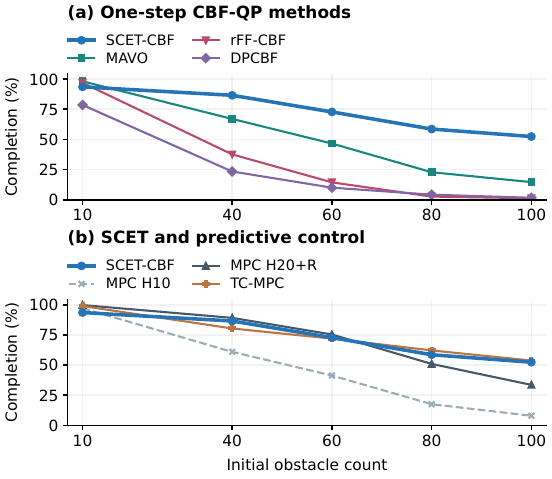}
\caption{Passage through denser encounters. (a) One-step CBF-QP designs.
(b) SCET and predictive control. Both use 500 shared scenarios per obstacle
count and a 60 s deadline; rFF and TC use the common-bicycle adaptations.}
\label{fig:primary}\end{figure}

On 1390 shared successes, median minimum clearance is 0.168/0.028 m
and steering variation 5.27/48.21 rad for SCET/enhanced MPC.
MPC arrives sooner (4.75 versus 13.50 s) along shorter paths; full quality
and complementary coverage appear in the technical companion. Enhanced MPC's 754 E
outcomes include 694 infeasibility reports and 60 other solver exits.
S/E end navigation; Sec.~V-F examines subsequent motion.

\subsection{Which guidance choices improve passage?}
Guiding the available commands produces the main gain. Segment protection
alone completes 224 primary tasks, adding braking gives 295, and adding
guidance gives 1820. With guidance, braking adds +5.1 pp [2.8,7.3] over
its removal (1693 successes). We next fix the hard layer and examine guidance.

\begin{figure*}[t]\centering
\includegraphics[width=.98\textwidth]{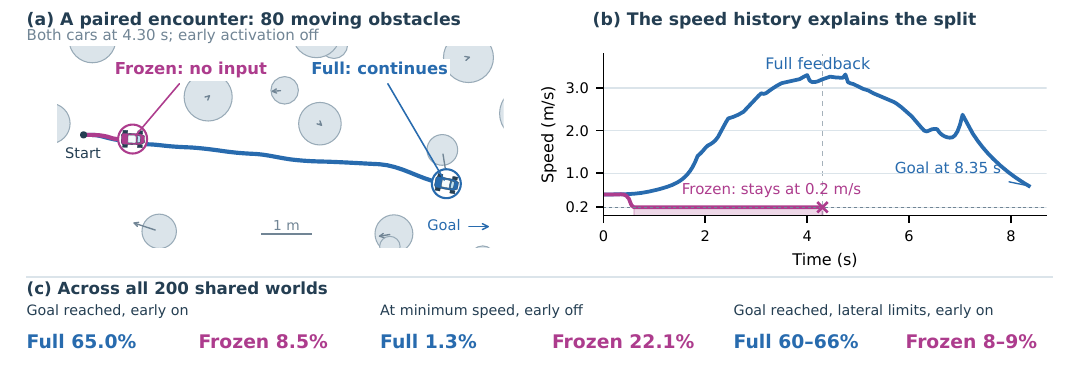}
\caption{Speed feedback changes the selected motion. Full retains the
allowance's speed derivative; frozen suppresses it at the same allowance
value for a given state. (a,b) One paired world with early activation off:
positions are synchronized at 4.30 s; speed histories end at each arm's
termination. (c) Results over all 200 shared worlds. Minimum-speed
occupancy is the mean per-world fraction of queries at the speed floor.
The ranges cover lateral limits of $0.1g$, $0.25g$ and $0.5g$.}
\label{fig:guidance}\end{figure*}

\paragraph{Allowance from turning capacity}
Replacing SCET's frontal rule by the reproduced DPCBF implementation's
$\nu_D=(0.5\,\mathrm{m/s})D_\epsilon\sqrt{1.05^2-1}/r^c$, where
$D_\epsilon=\sqrt{d^2-(r^c)^2+\epsilon^2}-\epsilon$ and $\epsilon=0.1$ m,
gives 37 versus 364 completions on 500 paired worlds: +65.4 pp
[61.6,69.2] for turning calibration. This intervention replaces the frontal
rule and its derived feedback and activation quantities at fixed weights.
On 200 worlds without early activation, simpler allowances
$kD_\epsilon$, $k=0.25,0.5,1,2,4$ s$^{-1}$, complete 31/35/36/31/29;
$kv$, $k=0.5,1,2,4,8$, complete 10/31/24/23/22, versus calibrated 74.

\paragraph{Feedback and early activation}
The frozen variant evaluates the same allowance at a given state and
disables its ego speed derivative. Figure~\ref{fig:guidance}a,b shows the
result in one encounter: full feedback builds speed and passes through;
the frozen car stays slow until no command is returned.
Across 200 shared worlds, full/frozen feedback completes
74/31 worlds with early activation off and 130/17 with it on:
differences +21.5 pp [16.0,27.0] and +56.5 pp
[50.0,63.5]. Mean per-world speed-floor occupancy rises from 1.3\% to
22.1\% when feedback is frozen with early activation off, linking lost
passage to prolonged low-speed operation.

The distance allowance has no ego speed sensitivity and gives identical
full/frozen outcomes: 29/29 with early activation off and 10/10 with it on. Early activation improves
the full turning arm from 74 to 130; its feedback interaction is +35.0 pp
[26.5,44.0]. The derivative-dependent advance is recomputed in each cell,
so the default gain includes this coupling.

\subsection{Guidance under lateral-acceleration limits}
Guidance should follow the response the vehicle can use.
We test saturation by imposing the steady turning bound~\cite{polack2018}:
\begin{equation}
\frac{v^2|\beta|\sqrt{1+\beta^2}}{\ell_r}\le a_{n,\max}.
\label{eq:cap}
\end{equation}
This defines a speed-dependent steering bound $\bar\beta(v)$. SCET uses
it in both the hard set and the turning proxy, including its derivative.
As the cap becomes active, $a_{\rm eff}$ and its allowance approach
saturation. The guidance therefore tracks the reduced turning response.
The six controllers in Table~\ref{tab:motion} receive the same bound. At $0.1g/0.25g/0.5g$,
full/frozen feedback completes 132/18, 131/16 and 120/17 of 200 worlds
(Fig.~\ref{fig:guidance}c). Enhanced MPC completes 80/106/132;
without a cap, SCET/enhanced MPC complete 130/144.
Keeping SCET's original uncapped allowance gives 83/107/125 under
the three caps: matching the allowance to the available response adds
24.5/12.0/$-2.5$ pp. The first two gains retain Holm-adjusted $p<0.05$.
At $0.1g$, the cap narrows steering on \CapRestricted\% of SCET's executed
steps. Predictive methods impose it at each predicted state using signed,
normalized inequalities.

\subsection{Bounded motion and execution errors}
Changing obstacle motion tests imperfect prediction. Step~1 adds known
ego and obstacle displacement bounds to its radius over the whole sample.
In one law, Cartesian accelerations are uniform on
$[-0.4,0.4]$ m/s$^2$ after control, with velocity projected onto the
1.2 m/s ball before position advances; the segment bound is
$\sqrt2(0.4)\dt^2$ m. In the other, velocity redraws within that ball occur
at rate 0.25 s$^{-1}$, giving $2(1.2)\dt$ m. All realized segments are checked.
Each of six methods runs with original radii and the same displacement
inflation on 200 shared worlds per law. Table~\ref{tab:motion} shows the
inflated arms. SCET completes 120/100 tasks with zero recorded contacts;
enhanced MPC completes 69/101. The acceleration-law difference is
25.5 pp [17.5,33.0]; the jump-law difference is $-0.5$ pp [$-8.0$,6.5].

\begin{table}[t]\centering\footnotesize
\caption{Changed obstacle motion: 200 paired scenarios per law, 60 s. Each method uses the same one-step displacement bound in its obstacle radius. Entries are S/C/E/T.}
\label{tab:motion}\setlength{\tabcolsep}{4pt}
\begin{tabular}{lrr}\toprule
Controller&Acceleration&Velocity jumps\\\midrule
SCET-CBF & 120/0/80/0 & 100/0/100/0\\
MAVO & 104/0/96/0 & 52/0/148/0\\
MPC, $H10$ & 49/0/151/0 & 66/0/134/0\\
MPC, $H20+R$ & 69/0/131/0 & 101/0/99/0\\
TC-MPC adaptation & 138/62/0/0 & 107/91/0/2\\
DPCBF analytic & 38/0/161/1 & 21/0/179/0\\
\bottomrule\end{tabular}\end{table}

Actuator dynamics delay the commanded response. Four profiles combine a 1 rad/s
wheel-rate limit, 20 m/s$^3$ jerk limit, 100 ms first-order lag and 50/100 ms command delay
(Table~\ref{tab:actuator}); each uses the same 100 worlds.
The execution plant has $\dot p=v e_{\theta+\beta_r}$,
$\dot\theta=v\sin\beta_r/\ell_r$ and
$\beta_r=\arctan(\tfrac12\tan\delta)$. Servos are integrated analytically,
with ten RK4 pose substeps per sample. Known servo parameters and queued
commands give endpoint ranges bounding $|\beta_r|\le B$ and $|\beta_r-\beta_c|\le E$.
For $h=\dt$ and $\Omega=\min(v_{\max},v_0+a_{\max}h)\sin B/\ell_r$,
integrating planning/execution velocity error gives
\begin{equation}
\rho_{\rm ego}=\tfrac12a_{\max}h^2+
 v_0(\tfrac12\Omega h^2+\tfrac12B^2h+Eh).
\label{eq:actuator-tube}
\end{equation}
Across 400 paired profile/world executions, including SCET's path bound
reduces contacts from 50 to zero, with completion changing from 179 to 167.
Its largest
path-error/bound ratio is 0.914. The six-method comparison uses identical
profiles and bound information. Whole-seed intervals retain all five
densities and four profiles together.

\begin{table}[t]\centering\footnotesize
\caption{Execution errors: 100 shared scenarios across four actuator configurations (400 runs per arm), 60 s. Entries are S/C/E/T.}
\label{tab:actuator}\setlength{\tabcolsep}{4pt}
\begin{tabular}{lrr}\toprule
Controller & Nominal radius & With path bound\\\midrule
SCET-CBF & 179/50/170/1 & 167/0/232/1\\
MAVO & 85/32/283/0 & 82/14/304/0\\
MPC, $H10$ & 24/16/354/6 & 15/0/377/8\\
MPC, $H20+R$ & 23/37/335/5 & 18/0/376/6\\
TC-MPC adapted & 182/214/0/4 & 192/199/0/9\\
DPCBF analytic & 97/5/298/0 & 99/1/300/0\\
\bottomrule\end{tabular}\end{table}

\paragraph{Optional post-task continuation}
The optional continuation retains a safety plan after task termination.
On 200 separate worlds with global constant-velocity predictions, the
ideal plant and zero computation delay, 49 short turn/speed maneuvers
connect to verified straight tails (Sec.~IV-C). All starts are admitted;
91 reach the goal before the first required fallback, and all runs retain
a valid continuation.

\subsection{Limitations and failure analysis}
The primary 670 SCET exhaustions comprise 513 empty steering sets and
157 hard-intersection losses. Continuing with straight braking to the
speed floor gives 651 later contacts under the constant-velocity model;
reaching E supplies no stopping certificate. The primary braking margin is nonnegative on 13.7\% of executed steps
and a braking row binds on 76.7\%. Outside that margin set it requests
recovery. Theorem~\ref{thm:pairwise} concerns a continuous pairwise system;
at the speed floor, its symmetric-box braking choice can conflict with
$a\ge0$, even for a single obstacle. Segment protection supplies the
sampled execution check.

At $N=10$, SCET/MAVO/MPC $H10$ complete 468/491/486; removing SCET's braking
row raises its count to 497 and reduces $N=100$ completion from 262 to
171. The braking layer therefore has a density-dependent cost and benefit.
Exact and tangent steering constructions with the same fixed 5\% buffer
complete 178/179 of 250 worlds, showing no completion gain from exactness.
The turning proxy shapes a soft cost: for static frontal encounters,
the zero set of $h/v$ is independent of speed. The near-obstacle $h=-1$ branch has zero input sensitivity; its
unsmoothed join can produce large numerical derivatives.
At $v=v_{\min}$, writing $\kappa=s_0/v$, the frontal central difference
is $\kappa/2-1$; the right derivative is $\kappa-1$.

MPC termination statuses describe local solver outcomes. In the three
capped TC arms, 439 of 461 E outcomes reach the iteration limit.
The common single-step inflation at prediction nodes provides no
full-horizon robust tube. Primary runs exclude computation delay;
enhanced MPC and TC exceed 50 ms in 632/894 and 1455/2709 timed calls.
Lateral limits, motion errors and actuator response are evaluated separately.
Tire transients, a shared friction budget, uncertain sensing
and uncalibrated servos require validation on a measured dynamic plant.
Future continuation certificates require predictions valid over their
entire execution. Goal entry and exhaustion issue no parking command.

\section{Conclusion}
SCET-CBF makes a one-step QP account for braking's two effects through a
turning-based soft allowance. Paired studies link improved dense passage
to speed feedback and early activation. Segment checks protect accepted
motion; lateral limits and bounded errors define the tested range.
Future work includes physical tests on an Ackermann robot, measured
steering dynamics, richer obstacle prediction and safe continuation.

\section*{Acknowledgments}
ChatGPT and Claude assisted text and presentation refinement (I to VI),
algebra verification and code refinement (III to V), and statistical and plotting script
improvements (V). The authors revised and reviewed the text, proofs,
code, records, and references.

\end{document}